\documentclass[sigconf]{acmart}

\newtheorem{remark}{Remark}
\newtheorem{example}{Example}
\usepackage[dvipsnames]{xcolor}
\usepackage[most]{tcolorbox}
\usepackage{amsmath}
\usepackage{multirow}
\usepackage{subcaption}
\usepackage[normalem]{ulem}
\usepackage{stfloats}

\definecolor{riskhigh}{HTML}{B4413C}
\definecolor{risklow}{HTML}{4A7C59}
\newcommand{\hi}[1]{\textbf{\textcolor{riskhigh}{#1}}}
\newcommand{\lo}[1]{\textbf{\textcolor{risklow}{#1}}}
\definecolor{cardteal}{HTML}{2A9D8F}
\definecolor{riskhigh}{HTML}{B4413C}  
\definecolor{risklow}{HTML}{4A7C59}   

\newtcolorbox{formulacard}{
  enhanced,
  colback=cardteal!6,
  colframe=cardteal,
  boxrule=0.6pt,
  arc=4pt,
  left=8pt, right=8pt, top=2pt, bottom=2pt,
  drop shadow={black!10},
  title=\textbf{Learned model},
  fonttitle=\bfseries\small, coltitle=white,
  fontupper=\small,
  coltitle=white,
  colbacktitle=cardteal,
}

\newtcolorbox{rulecard}{
  enhanced, rounded corners, arc=4pt,
  colback=cardteal!6, colframe=cardteal, boxrule=0.6pt,
  left=10pt, right=10pt, top=2pt, bottom=2pt,
  drop shadow={black!10},
  fonttitle=\bfseries\small, coltitle=white,
  fontupper=\small,
  title=Readable rule,
}

\newtcolorbox{rulecardsmall}[1]{
  enhanced, rounded corners, arc=3pt,
  colback=blue!3, colframe=blue!30, boxrule=0.7pt,
  left=8pt, right=8pt, top=5pt, bottom=5pt,
  fonttitle=\bfseries\small, coltitle=blue!50!black,
  colbacktitle=blue!8, title=#1,
  before skip=4pt, after skip=4pt,
}

\definecolor{cardteal}{HTML}{2A9D8F}
\definecolor{riskhigh}{HTML}{B4413C}
\definecolor{risklow}{HTML}{4A7C59}

\newtcolorbox{modelcard}[1]{
  enhanced,
  colback=cardteal!6,
  colframe=cardteal,
  boxrule=0.6pt,
  arc=4pt,
  left=8pt, right=8pt, top=2pt, bottom=2pt,
  title=\textbf{#1},
  fonttitle=\bfseries\small,
  coltitle=white,
  colbacktitle=cardteal,
  fontupper=\footnotesize,
  before upper={
  \setlength{\abovedisplayskip}{2pt}
  \setlength{\belowdisplayskip}{2pt}
  \setlength{\abovedisplayshortskip}{2pt}
  \setlength{\belowdisplayshortskip}{2pt}
},
}

\newtcolorbox{rulecard2}[1]{
  enhanced,
  colback=cardteal!6,
  colframe=cardteal,
  boxrule=0.6pt,
  arc=4pt,
  left=8pt, right=8pt, top=3pt, bottom=3pt,
  title=\textbf{#1},
  fonttitle=\bfseries\small,
  coltitle=white,
  colbacktitle=cardteal,
  fontupper=\footnotesize,
}
\AtBeginDocument{%
  }

\setcopyright{none}

\acmConference[Preprint]
{Submitted manuscript}
{2026}
{}
\renewcommand\footnotetextcopyrightpermission[1]{}

\begin{document}

\title{How Simple Can It Get? From Interpretable Equations to Readable Rules for Financial Decision Making}

\author{Adia Lumadjeng}
\orcid{1234-5678-9012}
\correspondingauthor
\email{a.c.lumadjeng@uva.nl}
\affiliation{%
  \institution{University of Amsterdam}
  \city{Amsterdam}
  \country{the Netherlands}
}

\author{Ilker Birbil}
\email{s.i.birbil@uva.nl}
\affiliation{%
  \institution{University of Amsterdam}
  \city{Amsterdam}
  \country{the Netherlands}
}

\author{Erman Acar}
\email{e.acar@uva.nl}
\affiliation{%
  \institution{University of Amsterdam}
  \city{Amsterdam}
  \country{the Netherlands}
}


\begin{abstract}
  In regulated domains such as finance, a model that cannot be explained cannot be deployed, yet many interpretable classifiers defeat their own purpose by producing formulas with dozens of features that no regulator could read. We take the reverse direction. Starting from an interpretable classifier expressed as a single equation over the input features, we progressively simplify it into more readable forms, including a pruned monomial, a directional if--then rule, and the integer scorecards and tallies that finance already deploys. Because the equation is itself the predictive model rather than a post-hoc explanation we can directly quantify what is lost under each simplification. Across four financial datasets, we find that pruning is nearly free and that fidelity can erode faster than predictive performance, allowing simpler rules to remain effective classifiers without faithfully reproducing the original model. A human assessment shows that simplification improves perceived readability, while preferences for different representations vary by professional background. Beyond measuring these losses empirically, we show that some can be anticipated from the original model: we derive a bound on the change caused by pruning and predict how faithfully a rule retaining only the direction of each feature's effect preserves the original ranking.
\end{abstract}

\ccsdesc[300]{Computing methodologies~Supervised learning}
\ccsdesc[300]{Applied computing~Business rules}
\ccsdesc[100]{Applied computing~Business process management}
\ccsdesc[500]{Computing methodologies~Machine learning}

\received{2 August 2026}

\maketitle

\section{Introduction}
When a lender rejects an application or a payment system freezes a transaction, someone is eventually owed a reason. In financial decision-making this is not only good practice but often a regulatory requirement: decisions must be explainable to a customer, an auditor, or a supervisor in terms a person can follow. Yet producing a useful explanation remains difficult. Credit default and fraud are typically imbalanced classification problems, where the events of interest form a small minority of observations. At the same time, interpretability is often treated as a structural property of a model, while the practical question is whether the resulting explanation can actually be understood and used by a human decision maker.

Interpretable-by-design classifiers address part of this problem by making the reasoning accessible from the model itself. One such classifier is ECSEL \citep{lumadjeng2026ecsel}, which learns a \emph{signomial} equation, i.e., a sum of power-law terms, that serves simultaneously as classifier and an explanation. In this work, we focus on its simplest single-term form: a \emph{monomial}, consisting of a product of input features raised to fitted exponents, and passed through a sigmoid to obtain the predicted probability. This monomial is the \textit{interpretable} equation we study throughout this paper. However, transparency by design does not guarantee readability in practice. Trained to predict loan default on a dataset, ECSEL produces the equation in Figure~\ref{fig:motivating-example}(a): a product over all forty-two features. Nothing in it is hidden, yet the resulting explanation is cumbersome to inspect and communicate. The very structure that makes the model transparent can therefore become a source of complexity when the equation grows large.

Figure~\ref{fig:motivating-example}(b) shows the same learned model after simplification into a directional rule over seven retained features. The resulting representation is more compact, but achieves this by discarding information: most features are removed, and the retained exponent magnitudes are replaced by their directions. The two panels therefore expose a tension that model transparency alone does not resolve: making an explanation easier to read requires deciding which parts of an already-interpretable model can be removed.

\begin{figure}[h]
\centering
\begin{minipage}{\columnwidth}
\centering
\begin{formulacard}

$
\text{risk}(x) ={} 5.91 \cdot
x_{\text{int\_received}}^{\,1.19} 
x_{\text{last\_pymt}}^{-0.85} 
x_{\text{time\_pymnt}}^{\,0.86}\cdot \ldots \footnotesize\text{(38 more features)} 
$

\end{formulacard}
{\small (a) The fitted monomial: a product over all 42 features.}
\end{minipage}

\vspace{0.7em}

\begin{minipage}{\columnwidth}
\centering
\begin{rulecard}
Predict \textit{default} when \textbf{total interest received} and \textbf{time since last payment} are \hi{high}, 
\textbf{total payment} and \textbf{last payment} are \lo{low}, 
\end{rulecard}
{\small (b) The same model as a readable rule over seven features.}
\end{minipage}

\caption{The same loan-default classifier as the full 42-feature monomial (a) and a simplified seven-feature directional rule (b).}
\label{fig:motivating-example}
\end{figure}

The interpretability literature has largely approached this problem in the forward direction: constructing models that remain transparent while retaining predictive performance \citep{lou2013accurate,letham2015interpretable}. We study the reverse direction. Given an already-interpretable model, \emph{how far can its representation be simplified while preserving predictive performance and fidelity to the original model, while becoming easier for people to understand and use?} This distinction is particularly relevant for interpretable-by-design approaches such as ECSEL. The original work evaluates the fitted model at full size while using pruning to produce more legible displayed equations \citep{lumadjeng2026ecsel}. It does not evaluate how these simplified representations perform as classifiers or how faithfully they preserve the behavior of the fitted model. We make this transition explicit. Starting from a fitted single-term monomial, we construct controlled simplifications that alter different components of its learned structure: pruning weakly contributing features, reducing exponent magnitudes to their directions, and replacing continuous feature values by binary conditions. This yields representations ranging from a pruned monomial to a sign-only directional rule and point-based scorecard and tally. Each representation is evaluated as the classifier it displays: the rule shown to the reader is also the rule whose predictions are evaluated. This lets us distinguish its predictive performance, its fidelity to the original model, and its perceived readability and practical preference in our human assessment.

Our contributions are as follows:
\begin{enumerate}
\item \textbf{Controlled simplification of an interpretable classifier.}
We derive readable rules directly from a learned monomial, with each representation preserving a different subset of the information encoded in the original interpretable equation.
\item \textbf{Quantifying the cost of simplification.}
Across financial datasets, we evaluate each derived representation as a classifier and against the original monomial, measuring predictive performance, rank fidelity, and calibration to identify which information can be removed at little cost.

\item \textbf{Predicting fidelity before simplification.}
For the directional rule, we derive a model-based prediction of rank fidelity from the fitted exponents and feature covariance, allowing the expected effect of sign-flattening to be assessed before evaluating the resulting rule.

\item \textbf{Assessing human readability.}
We conduct a human assessment with finance/risk professionals and AI/ML researchers to determine whether simplification translates into perceived ease and practical preference, and how these differ across professional backgrounds.\end{enumerate}

Finally, we compare post-hoc pruning with sparsity imposed during training through iterative hard thresholding. This serves as a robustness check on the extracted structure rather than a separate contribution: agreement between the two routes provides complementary evidence that the retained features are not solely an artifact of post-hoc pruning.

\section{Related Work}

Explainable artificial intelligence has traditionally addressed the lack of transparency of complex machine learning models through post-hoc explanations. Methods such as LIME \citep{ribeiro2016should} and SHAP \citep{lundberg2017unified} are widely used model-agnostic methods that explain individual predictions without modifying the underlying classifier. In credit scoring, \citet{chen2024imbalanced} find that their explanations become less stable as class imbalance increases. Because these explanations are produced post hoc, the explanation remains separate from the predictive model itself.

This limitation has motivated a complementary line of research on \emph{inherently interpretable} models, in which the predictor is designed to be transparent. Examples include generalized additive models \citep{lou2013accurate}, Bayesian Rule Lists \citep{letham2015interpretable}, Optimal Classification Trees \citep{bertsimas2017optimal}, and CORELS \citep{angelino2018learning}, with recent work extending interpretable trees to credit scoring \citep{tu2025oct}. Survey work distinguishes such ante-hoc approaches from post-hoc explanations \citep{dimarino2025survey}. Structural transparency, however, does not ensure readability: rule lists, trees, and additive models become harder to follow as they grow in size, depth, or interaction complexity. ECSEL \citep{lumadjeng2026ecsel} belongs to this family, representing the classifier through signomial score functions whose exponents admit interpretations. It faces the same tension, since the equation presented to a reader is pruned while the reported model is not.


Credit scoring provides a long-standing setting in which compact model representations are valued. Traditional scorecard development commonly bins characteristics, supporting model comprehensibility and review \citep{szepannek2022overview}. Modern approaches learn sparse integer scoring systems, including SLIM \citep{ustun2016slim}, RiskSLIM \citep{ustun2019risk}, and FasterRisk \citep{liu2022fasterrisk}, while \citet{chi2026risk} directly optimize such scores for decision net benefit. Structured lending models have also combined interpretable feature-group scores into an overall risk estimate \citep{chen2022holistic}. Model format and complexity can affect users' accuracy, response time, and confidence \citep{huysmans2011empirical}. Whereas these approaches learn compact representations directly from data, we derive them from an already fitted interpretable model. ECSEL \citep{lumadjeng2026ecsel} provides this starting point through its equation-based representation.

\section{Approach}
We first introduce the monomial classifier that serves as the starting point of our study, then derive a set of readable representations by systematically simplifying the information it contains.

We consider binary classification on tabular data, with features $x \in \mathbb{R}^m$ and labels $y \in \{0,1\}$, where $y=1$ denotes the minority risk class (e.g., \emph{default} or \emph{fraud}). Given training data $\{(x^{(i)},y^{(i)})\}_{i=1}^n$, we seek a scoring function that ranks positive instances above negatives and labels them after applying a decision threshold.

\paragraph{The Model.}
We build on ECSEL \citep{lumadjeng2026ecsel}, which learns a signomial equation as a joint classifier and explanation. A signomial is a sum of $K$ monomial terms, each a product of the features raised to fitted exponents. Because the exponents are real-valued, every feature is scaled to a fixed positive range $[\ell, u]$ with $0 < \ell$ before fitting; this scaling is fit on the training split alone. The model computes a single logit
\begin{equation}
z(x) = \sum_{k=1}^{K} \alpha_k \prod_{j=1}^{m} x_j^{\beta_{kj}}, \qquad P(y=1 \mid x) = \sigma\!\big(z(x)\big),
\end{equation}
where $\alpha_k \in \mathbb{R}$ are term coefficients, $\beta_{kj} \in \mathbb{R}$ are the fitted exponents, and $\sigma(z) = 1/(1+e^{-z})$ is the sigmoid. All tasks here are binary, so a single logit $z$ and a sigmoid suffice; the multiclass form is not needed, but is given in \cite{lumadjeng2026ecsel}. Parameters $\{\alpha_k, \beta_{kj}\}$ are learned by gradient descent, with sparsity regularization applied to the exponents. The number of terms $K$ controls model expressivity, ranging from a single monomial ($K{=}1$) to richer additive combinations of them ($K>1$). 

\paragraph{The Single Monomial.}
We restrict the general ECSEL formulation to its simplest case, $K{=}1$, yielding:
\begin{equation}\label{eq:loglinear}
z(x) = \alpha \prod_{j=1}^{m} x_j^{\beta_j}
\quad \xrightarrow{\log|\cdot|} \quad
\log|z(x)| = \log|\alpha| + \sum_{j=1}^{m} \beta_j \log x_j .
\end{equation}
In log-space, this monomial is linear: $|\beta_j|$ determines the feature $j$'s strength on the log-score, while the signs of $\beta_j$ and $\alpha$ determine its direction. For $\alpha>0$, positive exponents increase predicted risk and negative exponents decrease it; these directions reverse for $\alpha<0$.

A monomial explanation can therefore be viewed as containing three interpretable components: its \emph{support}, which determines which features participate; its \emph{direction}, determined by the sign of the exponents together with the sign of $\alpha$; and its \emph{magnitude}, given by the absolute exponent values. These components provide the basis for the controlled simplifications studied in Section~\ref{sec:forms}: each readable representation removes or preserves a different subset of the information contained in the original model.

\paragraph{Why the Monomial.}
We focus on ECSEL because it is an equation-learning method in which the learned expression serves explicitly as both the classifier and its explanation \citep{lumadjeng2026ecsel}. Its closed-form equations also support interpretability properties including global feature behavior, decision-boundary analysis, and local feature attribution. ECSEL therefore provides a natural setting for studying the gap between formal interpretability and practical readability, and for measuring what is preserved as an interpretable equation is progressively simplified.

Within ECSEL, we study the single monomial ($K=1$), its simplest equation form. This isolates the effect of simplifying the learned representation without introducing the additional structure of multiple terms. The same perspective may extend to other interpretable models, while multi-term signomials provide a natural next step within equation learning.

\subsection{Controlled Simplification of the Monomial} \label{sec:forms}

A learned monomial contains several kinds of information that contribute differently to its interpretation. Throughout this paper, we distinguish four components:

\begin{enumerate}
    \item \textbf{Feature support:} which features participate in the rule.
    \item \textbf{Feature values:} whether features are represented by their continuous values or by binary conditions.
    \item \textbf{Direction:} whether increasing a feature increases or decreases the risk, determined by the sign of its exponent.
    \item \textbf{Magnitude:} how strongly each feature influences the prediction, determined by the absolute value of its exponent.
\end{enumerate}

These components can be simplified independently, letting us isolate the cost of each simplification rather than treating readability as a single notion. Table~\ref{tab:components} summarizes the information retained by each representation.

\begin{table}[h]
\centering
\caption{Information retained by each representation.}
\label{tab:components}
\small
\begin{tabular}{lcccc}
\toprule
Form & Support & Values & Direction & Magnitude \\
\midrule
Full monomial    & All     & Cont. & \checkmark & Exact \\
Pruned monomial  & Reduced & Cont. & \checkmark & Exact \\
Directional rule & Reduced & Cont. & \checkmark & -- \\
Scorecard        & Reduced & Binary & \checkmark & Quantized \\
Tally            & Reduced & Binary & \checkmark & -- \\
\bottomrule
\end{tabular}
\end{table}

\subsubsection{Controlling Feature Support: the Pruned Monomial}
The first simplification reduces the feature support while leaving the remaining components unchanged: we retain $r$ features and set the remaining exponents to zero, so the pruned form is still a monomial over continuous values, with the same directions and the same exact magnitudes for the retained features. Only the number of participating features changes (Table~\ref{tab:components}).

This raises the question of which features to drop. We prune by exponent magnitude, keeping the largest $|\beta_j|$. In the log-space form of Eq.~\ref{eq:loglinear}, each feature contributes an additive term $\beta_j \log x_j$, so pruning by $|\beta_j|$ removes the terms with the smallest coefficients in absolute value. Because the features are scaled to a bounded positive range, each deleted term can be bounded in terms of its exponent, yielding a bound on the change in the log-score.

\begin{remark}[Worst-case perturbation under pruning]
\label{rem:prune-bound}
Let $S$ denote the set of features removed by pruning. Define the full and pruned log-scores as
\[
    s_{\mathrm{full}}(x):=\log|\alpha|+\sum_{j=1}^{m}\beta_j\log x_j, \quad s_{\mathrm{pruned}}(x):=\log|\alpha|+\sum_{j\notin S}\beta_j\log x_j.
\]
With features scaled to $[\ell,u]$, $0<\ell$, and $L:=\max\{|\log\ell|,\,|\log u|\}$, the log-scores before and after pruning satisfy
\[
\bigl|s_{\mathrm{full}}(x)-s_{\mathrm{pruned}}(x)\bigr|
= \Bigl|\sum_{j\in S}\beta_j\log x_j\Bigr|
\le L\sum_{j\in S}|\beta_j|,
\qquad \text{for every } x .
\]
\end{remark}

The bound follows from the triangle inequality and $|\log x_j| \le L$. Its content is a selection criterion rather than a performance guarantee: the right-hand side depends on the discarded exponents only through $\sum_{j \in S}|\beta_j|$, so among all ways of discarding $m-r$ features, dropping the smallest magnitudes minimizes the worst-case perturbation of the log-score. This motivates exponent magnitude as the pruning criterion, but does not guarantee preservation of predictive performance or ranking. We therefore choose $r$ on validation data as the smallest value whose PR-AUC remains within a fixed tolerance of the full monomial (see Section~\ref{sec:setup}).


\subsubsection{Binarizing Feature Values: Scorecard and Tally Rules}
The second simplification changes how the retained features are represented. The pruned monomial preserves the continuous feature values, whereas scorecards replace each feature with a binary risk condition. Continuous measurements are therefore reduced to a single yes/no statement, while the feature support and the direction of each effect are preserved. This follows the logic of traditional credit and clinical scoring systems, where an applicant accumulates points or meets conditions until a decision threshold is reached.


Each retained feature is converted into a binary condition with cutpoint $c_j$, set to the feature's training-set median (the midpoint for binary features). The signs of $\beta_j$ and $\alpha$ determine whether larger or smaller values indicate greater risk, and thus whether the condition is $x_j \ge c_j$ or $x_j \le c_j$. These thresholds are fixed rather than optimized to isolate the effect of replacing continuous values by binary conditions rather than learning a new scorecard. Cutpoint selection is a separate scorecard-modelling choice in credit risk \citep{szepannek2022overview}.

From these binary conditions we derive two additive representations. The first is the \emph{tally rule}, which treats every condition equally and counts how many are satisfied. Let $\mathbf{1}_j(x)$ denote the indicator that instance $x$ satisfies condition $j$. The tally rule is given by
\begin{equation}\label{eq:tally}
T(x)=\sum_{j=1}^{r}\mathbf{1}_j(x),
\qquad
\text{flag when }T(x)\ge t.
\end{equation}
where $t$ is the threshold on the number of satisfied conditions.

The second is the \emph{scorecard}, which retains a coarse notion of feature importance by assigning each condition an integer point value. Exponent magnitudes are rescaled so that the strongest retained condition receives $P_{\max}$ points, with every retained condition worth at least one:
\begin{equation}
    p_j=\max\left\{1,\;\operatorname{round}\left(P_{\max}\frac{|\beta_j|}{\max_k |\beta_k|}\right)\right\}.
\end{equation}
The floor keeps every feature in the support visible on the card: without it, a condition whose exponent is small relative to the largest would be assigned zero points and effectively removed from the scorecard. The resulting score is
\begin{equation}
\label{eq:scorecard}
V(x)=\sum_{j=1}^{r}p_j\,\mathbf{1}_j(x), \qquad \text{flag when }V(x)\ge v.
\end{equation}
where $v$ is the threshold on the total number of points.

The two representations differ only in how much magnitude information they retain: the tally is the scorecard with every point value set to one. The scorecard preserves relative importance approximately through integer points, whereas the tally removes it altogether. Both, however, discard the continuous feature values that remain available to the pruned monomial.

\begin{example}[Scoring an Applicant.]
Suppose pruning leaves the three-feature monomial
\begin{equation}
    z(x)=c\,x_{\mathrm{utilization}}^{+1.4}x_{\mathrm{months\_employed}}^{-0.9} x_{\mathrm{age}}^{-0.4}, \qquad c > 0,
\end{equation}
for predicting credit-card default. The signs give three conditions: high utilization, low employment, and low age. Consider an applicant with high utilization, low employment, and not low age. The corresponding indicators are therefore
\[
(\mathbf{1}_{\mathrm{utilization}},
\mathbf{1}_{\mathrm{months\_employed}},
\mathbf{1}_{\mathrm{age}})
=(1,1,0).
\]
With $P_{\max}=5$, rescaling the exponent magnitudes $(1.4,0.9,0.4)$ relative to the largest gives scorecard points $(5,3,1)$. The two representations therefore score the applicant 
\begin{align}
    T(x)&=1+1+0=2\ge t=2,\\
    V(x)&=5(1)+3(1)+1(0)=8\ge v=6.
\end{align}

Both rules therefore flag the applicant as risky. They need not agree in general: the tally counts every satisfied condition equally, whereas the scorecard weights them by their relative importance.
\end{example}

\subsubsection{Removing Effect Magnitudes: The Directional Rule}
\label{sec:directional}
The final simplification removes only the effect magnitudes while leaving the remaining components unchanged. Unlike the scorecard and tally, it preserves the continuous feature values and their multiplicative combination, but discards the relative strength of each feature. This is achieved by replacing every exponent with its sign, giving the directional log-score
\begin{equation}
  s_{\mathrm{dir}}(x) \;=\; \sum_{j=1}^{r}
  \operatorname{sign}(\beta_j)\,\log x_j .
  \label{eq:directional}
\end{equation}
Viewed in log-space, the directional rule replaces the coefficient vector $\beta$ with its sign vector. The feature support and directions are preserved, and only the coefficient magnitudes are removed. The intercept $\log|\alpha|$ is irrelevant to the ranking and therefore omitted. This is the most aggressive of our simplifications, and the one whose cost is least predictable from the form alone. Intuitively, if all retained features contribute equally, replacing their exponent magnitudes by one should have little effect on the ranking, whereas if a few features dominate the monomial, flattening the exponent magnitudes may substantially alter it.


The following proposition characterizes the similarity between the full and directional scores after exponent flattening. This similarity can be expressed as a Pearson correlation computed from the fitted exponents and feature covariance. Our fidelity measure, however, is the tie-corrected Kendall's $\tau_b$ \citep{lindskog2003kendall}. Under elliptical feature distributions, Greiner's classical relation \citep{greiner1909} links Pearson correlation to Kendall's $\tau$. We therefore use this relation to predict rank fidelity, noting that $\tau$ and $\tau_b$ can differ when ties are frequent.

\begin{proposition}[Fidelity after exponent flattening]
\label{prop:fidelity}
Write $w = (\log x_1, \dots, \log x_m)$ and let $\Sigma$ denote its covariance. Up to constants, the full and directional log-scores are the linear functions $s_{\mathrm{full}} = \beta^{\top} w$ and $s_{\mathrm{dir}} = d^{\top} w$, where $d_j = \operatorname{sign}(\beta_j)$ on the retained set $S$, $|S| = r$, and $d_j = 0$ otherwise. Their correlation is
\begin{equation}
  \rho \;=\;
  \frac{\beta^{\top} \Sigma\, d}
       {\sqrt{\beta^{\top} \Sigma\, \beta}\,
        \sqrt{d^{\top} \Sigma\, d}} ,
  \label{eq:rho}
\end{equation}
which holds without distributional assumptions. If $w$ is elliptically distributed, Greiner's relation \citep{greiner1909} links this correlation to the rank fidelity between the two scores:
\begin{equation}
  \tau \;=\; \tfrac{2}{\pi}\arcsin(\rho).
  \label{eq:greiner}
\end{equation}
\end{proposition}

\begin{proof}[Proof sketch]
For any fixed vectors $a,b$,
\[
\operatorname{Cov}(a^\top w,b^\top w)=a^\top\Sigma b,
\qquad
\operatorname{Var}(a^\top w)=a^\top\Sigma a.
\]
Substituting $a=\beta$ and $b=d$ into the definition of the Pearson correlation between $s_{\mathrm{full}}$ and $s_{\mathrm{dir}}$ yields \eqref{eq:rho}, which requires only that the second moments of $w$ exist. Since $(s_{\mathrm{full}},s_{\mathrm{dir}})$ is a linear transformation of $w$, it is elliptically distributed whenever $w$ is. Equation~\eqref{eq:greiner} therefore follows directly from Greiner's classical relation.
\end{proof}

Equation~\eqref{eq:rho} makes this intuition precise: $\rho$ is high when replacing the fitted exponents by their signs leaves the two score directions well aligned under the observed feature covariance, and low when exponent magnitudes or feature dependencies make that flattening consequential. The proposition therefore yields a prediction of directional-rule fidelity directly after training: $\beta$ and $d$ follow from the fitted exponents and $\Sigma$ from the training data. No held-out data or evaluation of the directional rule is required. Section~\ref{sec:pred-fidelity} compares the resulting predicted Kendall's $\tau$ with observed fidelity.

\subsection{Sparsity During Training}
Every simplification so far starts from a fitted monomial with all $m$ exponents non-zero. This is not incidental: ECSEL regularizes the exponents with an $\ell_1$ penalty, which shrinks small coefficients but under gradient-based optimization does not reliably drive them to exact zeros, so a compact rule always requires a separate reduction step afterwards. As an alternative, we use iterative hard thresholding (IHT) \citep{BlumensathDavies2008IT}, which retains only the $s$ largest exponents after each gradient step. This provides an independent route to sparsity for assessing whether the extracted structure depends on post-hoc pruning.

\subsection{Human Assessment of Readable Forms}\label{sec:human}
The simplifications are intended to produce representations that are not only formally simpler, but easier to understand and use. We therefore conduct an anonymous human assessment using the \textsc{Loan} and \textsc{FraudEcom} classifiers. For each dataset, participants are presented with the pruned monomial, directional rule, scorecard, and tally derived from the fitted classifier, without information about predictive performance. They rate each representation's ease of understanding on a five-point Likert scale and select the representation they would prefer for practical use. This allows us to distinguish perceived readability from practical preference and examine whether these judgments differ by professional background.

\section{Experimental Setup}
\label{sec:setup}

\paragraph{Datasets.}
We evaluate on four public financial classification datasets, summarized in Table~\ref{tab:datasets}. They cover consumer credit default (\textsc{Loan} \citep{lendingclubdata}, \textsc{Default} \citep{yeh2009comparisons}) and fraud detection (\textsc{FraudEcom} \citep{fraudecommerce}, \textsc{Creditcard} \citep{dalpozzolo2015calibrating}), range over an order of magnitude in size, and span positive-class rates from $0.17\%$ to $22\%$. The datasets also differ in feature type: \textsc{Loan} and \textsc{Default} contain named continuous attributes, \textsc{FraudEcom} combines continuous variables with small integer counts, and \textsc{Creditcard} consists of anonymized principal components. As discussed in Section~\ref{sec:results}, these differences influence which simplifications remain effective.

\begin{table}[h]
\small
\caption{Financial datasets, ordered by class imbalance; positive-class rate is the PR-AUC base rate.}
\label{tab:datasets}
\begin{tabular}{lcccc}
\toprule
Dataset & $n$ & Features & Pos. rate & Domain \\
\midrule
\textsc{Creditcard} & 284{,}807 & 30 & 0.17\% & card fraud \\
\textsc{FraudEcom}  & 151{,}112 &  6 & 9.4\%  & account fraud \\
\textsc{Loan}       & 395{,}492 & 42 & 10.1\% & loan default \\
\textsc{Default}    &  30{,}000 & 26 & 22.1\% & credit default \\
\bottomrule
\end{tabular}
\end{table}

\paragraph{Preprocessing and evaluation.}
Each dataset is split into training, validation, and test sets (60/20/20, stratified by label), repeated over five random seeds. Following \citet{lumadjeng2026ecsel}, features are scaled to $[0.01,10.01]$, with the scaler fit on the training split only. Hyperparameters were optimized using Optuna on the training split, with the best configuration selected by cross-validation. Pruning level and all decision thresholds are chosen on the validation split. The test split is used only for the final evaluation. Unless stated otherwise, reported results are means over the five seeds.

\paragraph{Models and readable forms.}
Our model is the ECSEL monomial. From each fitted monomial we derive the four readable forms of Section~\ref{sec:forms}: the pruned monomial, directional rule, scorecard, and tally, and evaluate each as a classifier in its own right. The pruning level is selected as the smallest number of retained features $r$ whose validation PR-AUC is within 0.01 of the full monomial. We additionally compare post-hoc pruning with an in-training sparse variant based on IHT. Throughout, scorecards use a maximum of five points ($P_{\max}=5$), while scorecard and tally thresholds are selected on the validation split. 

\paragraph{Metrics.}
Because all datasets are imbalanced, we use the area under the precision--recall curve (PR-AUC) as the primary performance metric and report the corresponding base rate for reference. To measure how faithfully a simplified form reproduces the full monomial, we report Kendall's $\tau_b$ between their scores,
\begin{equation}
    \tau_b =\frac{C-D}{\sqrt{(C+D+T_s)(C+D+T_f)}},
\end{equation}
where $C$ and $D$ denote the number of concordant and discordant pairs, and $T_s$ and $T_f$ account for ties in the simplified and full scores, respectively. For the directional rule, we additionally compare the observed $\tau_b$ with the value $\tau_\text{pred}$ predicted by Proposition~\ref{prop:fidelity}, computed solely from the fitted exponents and the training data.

Finally, we assess whether the scores produced by each representation can be converted into reliable probability estimates. Since the directional rule, scorecard, and tally do not natively output probabilities, we calibrate each representation on the validation set using isotonic regression and report its expected calibration error (ECE) on the test set,
\begin{equation}
    \mathrm{ECE} = \sum_{b=1}^{B} \frac{|b|}{n} \left| \mathrm{acc}(b)-\mathrm{conf}(b) \right|,
\end{equation}
where the test set is partitioned into $B$ equal-mass bins by predicted probability, $|b|$ denotes the number of instances in bin $b$, $\mathrm{acc}(b)$ its observed positive rate, and $\mathrm{conf}(b)$ its mean predicted probability. ECE therefore assesses the reliability of the recalibrated probabilities. Unless stated otherwise, all reported metrics are averaged over the five random seeds.

\paragraph{Human assessment.}
Of 36 respondents, 34 were included in the subgroup analysis: 14 with a finance/risk background (including three also reporting AI/ML research) and 20 AI/ML researchers. Two respondents reporting neither background were excluded.

\section{Results}\label{sec:results}
We first quantify the cost of each controlled simplification and examine when different representations succeed or fail across datasets. We next evaluate the prediction of Proposition~\ref{prop:fidelity}, assess the readability of the resulting explanations in a human study, and finally compare post-hoc pruning with in-training sparsity.

Figure~\ref{fig:ruleexamples} shows the resulting representations for \textsc{Loan} and \textsc{FraudEcom}; the remaining datasets undergo the same controlled simplifications but are omitted for brevity. These are also the representations presented to participants in the human assessment.

\begin{figure*}[t]
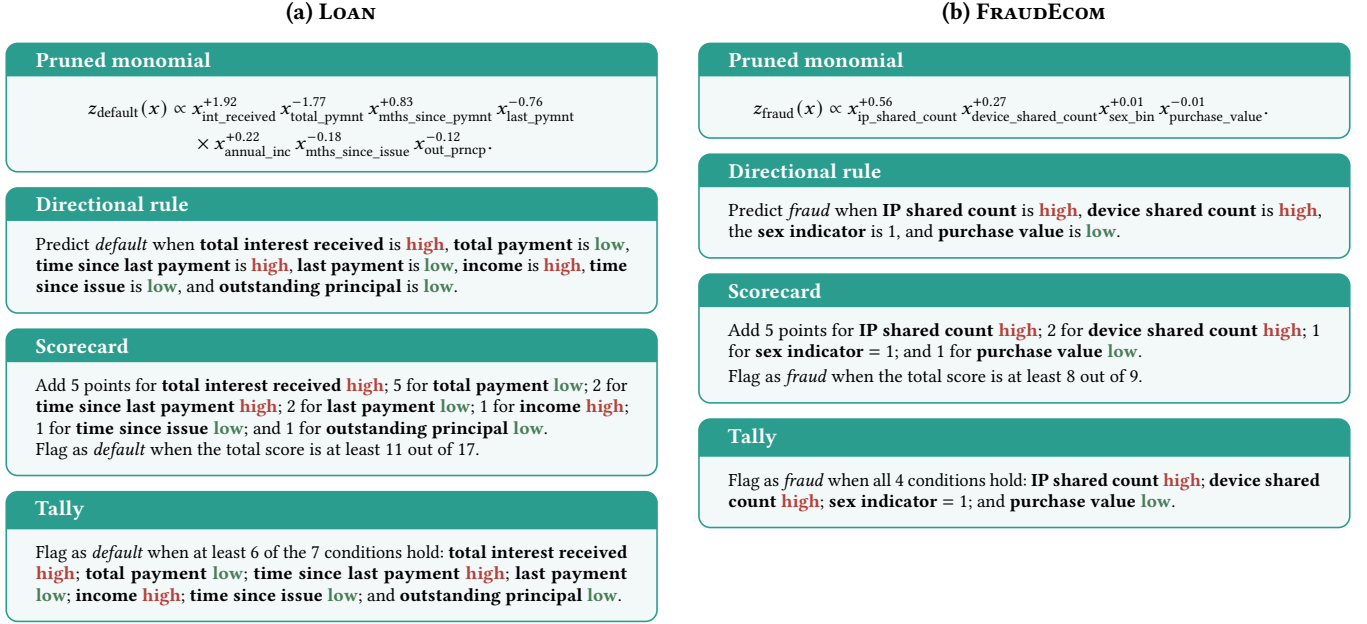

\centering

\begin{minipage}[t]{0.485\textwidth}
\centering
\textbf{(a) \textsc{Loan}}\\[2pt]

\begin{modelcard}{Pruned monomial}
\[
\begin{aligned}
z_{\text{default}}(x) \propto\;
  &x_{\text{int\_received}}^{+1.92}\,
   x_{\text{total\_pymnt}}^{-1.77}\,
   x_{\text{mths\_since\_pymnt}}^{+0.83}\,
   x_{\text{last\_pymnt}}^{-0.76}\\
  &{}\times
   x_{\text{annual\_inc}}^{+0.22}\,
   x_{\text{mths\_since\_issue}}^{-0.18}\,
   x_{\text{out\_prncp}}^{-0.12}.
\end{aligned}
\]
\end{modelcard}

\begin{rulecard2}{Directional rule}
Predict \textit{default} when
\textbf{total interest received} is \hi{high},
\textbf{total payment} is \lo{low},
\textbf{time since last payment} is \hi{high},
\textbf{last payment} is \lo{low},
\textbf{income} is \hi{high},
\textbf{time since issue} is \lo{low}, and
\textbf{outstanding principal} is \lo{low}.
\end{rulecard2}

\begin{rulecard2}{Scorecard}
Add $5$ points for \textbf{total interest received} \hi{high};
$5$ for \textbf{total payment} \lo{low};
$2$ for \textbf{time since last payment} \hi{high};
$2$ for \textbf{last payment} \lo{low};
$1$ for \textbf{income} \hi{high};
$1$ for \textbf{time since issue} \lo{low}; and
$1$ for \textbf{outstanding principal} \lo{low}.

Flag as \textit{default} when the total score is at least $11$ out of $17$.
\end{rulecard2}

\begin{rulecard2}{Tally}
Flag as \textit{default} when at least $6$ of the $7$ conditions hold:
\textbf{total interest received} \hi{high};
\textbf{total payment} \lo{low};
\textbf{time since last payment} \hi{high};
\textbf{last payment} \lo{low};
\textbf{income} \hi{high};
\textbf{time since issue} \lo{low}; and
\textbf{outstanding principal} \lo{low}.
\end{rulecard2}

\end{minipage}
\hfill
\begin{minipage}[t]{0.485\textwidth}
\centering
\textbf{(b) \textsc{FraudEcom}}\\[5pt]

\begin{modelcard}{Pruned monomial}
\[
\begin{aligned}
z_{\text{fraud}}(x) \propto\;
  &x_{\text{ip\_shared\_count}}^{+0.56}\,
   x_{\text{device\_shared\_count}}^{+0.27}
   x_{\text{sex\_bin}}^{+0.01}\,
   x_{\text{purchase\_value}}^{-0.01}.
\end{aligned}
\]
\end{modelcard}

\begin{rulecard2}{Directional rule}
Predict \textit{fraud} when
\textbf{IP shared count} is \hi{high},
\textbf{device shared count} is \hi{high},
the \textbf{sex indicator} is $1$, and
\textbf{purchase value} is \lo{low}.
\end{rulecard2}

\begin{rulecard2}{Scorecard}
Add $5$ points for \textbf{IP shared count} \hi{high};
$2$ for \textbf{device shared count} \hi{high};
$1$ for \textbf{sex indicator} $=1$; and
$1$ for \textbf{purchase value} \lo{low}.

\smallskip
Flag as \textit{fraud} when the total score is at least $8$ out of $9$.
\end{rulecard2}

\begin{rulecard2}{Tally}
Flag as \textit{fraud} when all $4$ conditions hold:
\textbf{IP shared count} \hi{high};
\textbf{device shared count} \hi{high};
\textbf{sex indicator} $=1$; and
\textbf{purchase value} \lo{low}.
\end{rulecard2}
\end{minipage}

\caption{Controlled simplifications of representative \textsc{Loan} and \textsc{FraudEcom} classifiers (seed 45). Each panel derives four representations from the same retained features.}
\label{fig:ruleexamples}
\end{figure*}

\subsection{The Cost of Controlled Simplification}
Table~\ref{tab:cost} reports predictive performance, fidelity, and recalibrated ECE for the controlled simplifications of Section~\ref{sec:forms}. PR-AUC base rates are 0.101 (\textsc{Loan}), 0.094 (\textsc{FraudEcom}), 0.221 (\textsc{Default}), and 0.002 (\textsc{Creditcard}). Across five seeds, variability is generally small for the full and pruned monomials but larger for some aggressive simplifications. We discuss several patterns below.

\begin{table}[h]
\centering
\small
\setlength{\tabcolsep}{3.5pt}
\caption{Predictive performance, fidelity, and recalibrated ECE
(mean over five seeds).}
\label{tab:cost}

\begin{tabular}{llccccc}
\toprule
Dataset & Metric & Full & Pruned & Directional & Scorecard & Tally \\
\midrule

\multirow{3}{*}{\textsc{Loan}}
& PR-AUC   & 0.886 & 0.885 & 0.642 & 0.329 & 0.300 \\
& $\tau_b$ & 1.000 & 0.846 & 0.568 & 0.206 & 0.196 \\
& ECE      & 0.001 & 0.001 & 0.001 & 0.002 & 0.002 \\
\midrule

\multirow{3}{*}{\textsc{FraudEcom}}
& PR-AUC   & 0.647 & 0.647 & 0.639 & 0.333 & 0.291 \\
& $\tau_b$ & 1.000 & 0.845 & 0.831 & 0.670 & 0.656 \\
& ECE      & 0.004 & 0.004 & 0.004 & 0.004 & 0.004 \\
\midrule

\multirow{3}{*}{\textsc{Default}}
& PR-AUC   & 0.380 & 0.377 & 0.381 & 0.303 & 0.334 \\
& $\tau_b$ & 1.000 & 0.820 & 0.744 & 0.375 & 0.530 \\
& ECE      & 0.014 & 0.016 & 0.016 & 0.014 & 0.016 \\
\midrule

\multirow{3}{*}{\textsc{Creditcard}}
& PR-AUC   & 0.677 & 0.675 & 0.621 & 0.018 & 0.017 \\
& $\tau_b$ & 1.000 & 0.698 & 0.614 & 0.259 & 0.406 \\
& ECE      & 0.000 & 0.000 & 0.000 & 0.000 & 0.000 \\
\bottomrule
\end{tabular}
\end{table}

\paragraph{Pruning is nearly free.}
Pruning selected under the validation tolerance generalizes with little predictive loss on the held-out test sets. The largest reduction in PR-AUC is 0.008 (\textsc{FraudEcom}), while \textsc{Loan} and \textsc{Creditcard} lose only 0.001 and 0.002, respectively. Fidelity remains high, with Kendall's $\tau_b$ between 0.70 and 0.85. Thus, substantial reductions in feature support can be achieved without materially degrading held-out predictive performance.

\paragraph{Binarizing feature values is costly.}
On \textsc{Loan}, PR-AUC falls from $0.885$ for the pruned monomial to $0.329$ for the scorecard, and $0.300$ for the tally. Similar reductions occur on \textsc{Creditcard}, where both point-based forms approach the dataset base rate. Unlike the other datasets, however, the \textsc{Creditcard} features are anonymized principal components rather than meaningful financial variables. Thresholding these latent components discards much of the information contained in their continuous values, making them particularly unsuitable for point-based rules. We return to this observation in the next subsection. The cost of binarizing feature values is therefore substantially larger than that of removing effect magnitudes, suggesting that, on these datasets, the continuous feature values carry more predictive information than the precise exponent magnitudes.


\paragraph{Predictive performance and fidelity capture different properties.}
The directional rule illustrates the distinction between predictive performance and fidelity. On \textsc{Default}, it essentially matches the predictive performance of the full monomial (PR-AUC $0.381$ versus $0.380$) despite a lower rank fidelity ($\tau_b = 0.744$). Conversely, on \textsc{Loan}, fidelity drops to $0.568$ while predictive performance remains well above the scorecard and tally. Fidelity can therefore erode faster than predictive performance, while the simplified rule remains an effective classifier.

\paragraph{Simplified scores can be reliably calibrated.}
From Table~\ref{tab:cost}, we see that after validation-set isotonic recalibration, ECE remains low across representations and datasets, showing that simplified scores can yield reliable probability estimates despite losses in ranking.


\subsection{When Do Different Simplifications Succeed?}
The results of Table~\ref{tab:cost} show that no representation is uniformly best. Instead, the success of a simplification depends on whether the information it removes is predictive for the underlying dataset. 

The two classifiers illustrate contrasting outcomes. On \textsc{Loan}, pruning reduces 42 features to seven with virtually no predictive loss, while removing effect magnitudes and especially binarizing continuous values is costly. In contrast, \textsc{FraudEcom} retains four features, two of which dominate the others by almost two orders of magnitude. Removing effect magnitudes therefore has little effect (PR-AUC $0.647$ to $0.639$), making the directional rule a simpler yet faithful summary. The remaining datasets reinforce this variation: on \textsc{Default}, the directional rule matches the full monomial closely (PR-AUC $0.381$ versus $0.380$) despite retaining ten to thirteen features, suggesting that useful signal is distributed across many variables. On \textsc{Creditcard}, binarizing the anonymized components reduces the point-based forms to near-base-rate performance, while the directional rule remains stronger. Thus, simplification is not a single ladder of increasingly weaker models: its cost depends on where the monomial's predictive information resides.

\subsection{Predicting Fidelity Before Simplification}\label{sec:pred-fidelity}
Proposition~\ref{prop:fidelity} predicts the rank fidelity of the directional rule from the learned monomial and training data, without constructing or evaluating the simplified rule. Table~\ref{tab:predfid} compares predicted and observed Kendall rank fidelity across five seeds.

\begin{table}[h!]
\centering
\small
\caption{Predicted and observed rank fidelity across five seeds.}
\label{tab:predfid}
\begin{tabular}{lcccc}
\toprule
Dataset &
$\rho$ &
$\tau_{\text{pred}}$ &
$\tau_b$ &
MAE \\
\midrule
\textsc{Loan}
& $0.732 \pm 0.260$
& $0.578 \pm 0.270$
& $0.568 \pm 0.248$
& $0.024$ \\

\textsc{Default}
& $0.911 \pm 0.027$
& $0.731 \pm 0.040$
& $0.744 \pm 0.039$
& $0.019$ \\

\textsc{Creditcard}
& $0.821 \pm 0.018$
& $0.614 \pm 0.019$
& $0.614 \pm 0.026$
& $0.008$ \\

\textsc{FraudEcom}
& $0.906 \pm 0.134$
& $0.759 \pm 0.165$
& $0.831 \pm 0.238$
& $0.236$ \\
\bottomrule
\end{tabular}
\end{table}

On \textsc{Loan}, \textsc{Default}, and \textsc{Creditcard}, the prediction is consistently accurate. Across the fifteen fitted models, the mean absolute error is $0.017$, with a maximum error of $0.039$. The prediction remains accurate over observed fidelities ranging from $0.25$ to $0.83$, providing empirical support for Proposition~\ref{prop:fidelity} on these datasets.

\textsc{FraudEcom} is the exception, with a substantially larger MAE of $0.236$. The failure is driven by variation across fits: in three seeds, $\rho\approx0.96$ corresponds to observed $\tau_b\approx0.98$, whereas in another $\rho=0.973$ but $\tau_b=0.440$. The dominant features are low-cardinality count variables, producing many tied scores. Pearson correlation is insensitive to this tie structure, whereas Kendall's $\tau_b$ accounts for it. Consequently, the Greiner conversion becomes inaccurate, reflecting a failure of its distributional assumptions rather than of the correlation formula itself.

Thus, the fidelity of the directional rule is generally predictable immediately after training. Among the simplifications considered here, it is the only representation whose expected agreement with the original classifier can be estimated before the simplified rule is produced.

\subsection{Human Readability and Preference}
\label{sec:results_human}

Participants evaluated the \textsc{Loan} and \textsc{FraudEcom} representations shown in Figure~\ref{fig:ruleexamples}. 
Objective comprehension was high across all four forms, with 92--100\% of responses correctly identifying how a stated feature condition affected the prediction. The main differences therefore concern perceived ease and practical preference rather than basic understanding. All three simplified rules were perceived as easier to understand than the pruned monomial on both datasets.

\begin{figure}[h!]
    \centering
    \includegraphics[width=0.85\columnwidth]{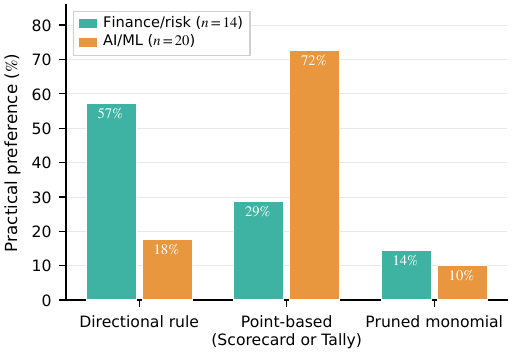}
    \caption{Practical preference by professional background, aggregated
    across \textsc{Loan} and \textsc{FraudEcom}.}
    \label{fig:human-preference}
\end{figure}

Aggregated across both datasets, finance/risk respondents selected the directional rule in $57\%$ of assessments, compared with $29\%$ for the point-based forms. AI/ML researchers showed the opposite pattern, selecting point-based forms in $72\%$ of assessments and the directional rule in $18\%$. The pruned monomial was rarely preferred by either group ($14\%$ and $10\%$, respectively). Thus, while simplification improves perceived ease, the representation preferred in practice differs by professional background.

\begin{figure}[h!]
    \centering
    \includegraphics[width=0.85\columnwidth]{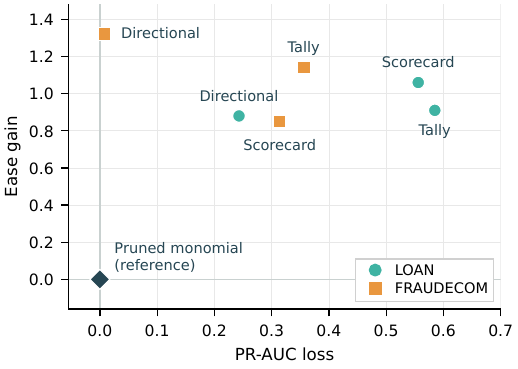}
    \caption{Gain in mean perceived ease versus loss in PR-AUC, both
    relative to the pruned monomial.}
    \label{fig:human-tradeoff}
\end{figure}

Preference alone, however, does not capture what is preserved by a simplification. Figure~\ref{fig:human-tradeoff} plots, relative to the pruned monomial, the increase in mean perceived ease against the loss in PR-AUC. The directional rule gains substantially in perceived ease at lower predictive cost than the point-based forms. This is particularly pronounced on \textsc{FraudEcom}, where PR-AUC changes only from $0.647$ to $0.639$. Participants were not shown predictive performance, so this is an ex-post comparison rather than a trade-off they were asked to make.

\subsection{In-Training Sparsity Recovers Similar Rules}
Finally, we examine whether the extracted rules depend on the post-hoc pruning procedure itself. Table~\ref{tab:iht} compares post-hoc pruning with an in-training sparse variant based on iterative hard thresholding (IHT), reporting the number of retained features $r$, their Jaccard overlap, and predictive performance.  Across all datasets, both approaches retain nearly identical feature sets and achieve comparable predictive performance. 
Feature selection is identical on \textsc{Loan}, \textsc{FraudEcom}, and \textsc{Creditcard}, and differs by only two of twelve features on \textsc{Default}, suggesting that the extracted rules primarily reflect learned data structure rather than pruning artifacts.

\begin{table}[h!]
\centering
\small
\caption{Post-hoc pruning versus in-training sparsity (IHT).}
\label{tab:iht}
\begin{tabular}{lccccc}
\toprule
Dataset &
$r_{\text{prune}}$ &
$r_{\text{IHT}}$ &
Jaccard &
PR-AUC$_{\text{prune}}$ &
PR-AUC$_{\text{IHT}}$ \\
\midrule
\textsc{Loan}       & 6  & 6  & 1.00 & 0.885 & 0.889 \\
\textsc{FraudEcom}  & 3  & 3  & 1.00 & 0.647 & 0.651 \\
\textsc{Default}    & 12 & 12 & 0.83 & 0.377 & 0.381 \\
\textsc{Creditcard} & 3  & 3  & 1.00 & 0.675 & 0.691 \\
\bottomrule
\end{tabular}
\end{table}

\section{Limitations and Future Work}
Our study is deliberately restricted to single monomials ($K=1$) and four financial decision-making datasets, and the findings may not extend directly to richer models or other domains. The readable forms likewise use fixed simplifications, such as median binarization, by design: our aim is to isolate the cost of removing specific information rather than optimize each representation for predictive performance. The results therefore characterize the cost of the proposed simplifications, not the best achievable performance of each rule form. Our human assessment is exploratory, with 36 respondents, and measures stated preferences rather than use in actual decisions. Future work could extend controlled simplification to multi-term signomials ($K>1$), other inherently interpretable models, and larger decision-based human evaluations.

\section{Conclusion}
In financial decision making, an interpretable model can still be too complex to communicate. We studied this gap by progressively reducing learned interpretable equations into pruned equations, directional rules, scorecards, and tallies, and measuring predictive performance, fidelity, and human perception at each stage. The reductions are not points on a single accuracy--readability trade-off. Removing weakly contributing features is nearly free, while removing effect magnitudes is often cheap and its effect on fidelity can be anticipated from the fitted model. Binarizing continuous feature values to derive point-based forms incurs the largest predictive losses. Importantly, predictive performance and fidelity can diverge: a simpler rule may remain an effective classifier without faithfully reproducing the original model. The human assessment confirms that simplification improves perceived ease, while practical preferences in our sample vary by professional background. Readable explanations are therefore better treated as measurable reductions of an interpretable model than assumed faithful because they are easier to read.

\section*{Ethics and Privacy Statement}
We use public financial benchmark datasets without personally identifiable information. While simplification may improve transparency, it can discard information relevant to accurate and equitable decisions and should therefore complement, not replace, domain validation, fairness assessment, and regulatory review.


\bibliographystyle{ACM-Reference-Format}
\bibliography{sample-base}

\appendix

\end{document}